\documentclass[runningheads]{llncs}

\usepackage{eccv}

\usepackage{eccvabbrv}

\usepackage{graphicx}
\usepackage{booktabs}

\usepackage[accsupp]{axessibility}  

\usepackage{hyperref}

\usepackage{orcidlink}

\usepackage{amsmath}
\usepackage{amssymb}
\usepackage{gensymb}
\usepackage[misc]{ifsym}

\usepackage[dvipsnames,table]{xcolor}
\usepackage{makecell}
\usepackage{multirow}
\usepackage{booktabs}
\usepackage{colortbl}
\usepackage{wrapfig}

\usepackage{twemojis}

\usepackage{float}

\usepackage{algorithm}
\usepackage{algpseudocode}
\usepackage{enumitem}

\usepackage{xspace}
\newcommand{\mymethod}{\texttt{CAESAR}\xspace}
\newcommand{\mymethodvivid}{\includegraphics[width=0.3cm]{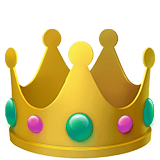}\hspace{0.05cm}\mymethod}

\newcommand{\matcherfull}{Dual-Cue Guided Re-Matching\xspace}
\newcommand{\matcherfullb}{\textbf{D}ual-\textbf{C}ue \textbf{G}uided \textbf{R}e-Matching\xspace}
\newcommand{\matcher}{DCRM\xspace}

\newcommand{\labelerfull}{Semantic-Geometric Label Mining\xspace}
\newcommand{\labelerfullb}{\textbf{S}emantic-\textbf{G}eometric \textbf{L}abel \textbf{M}ining\xspace}
\newcommand{\labeler}{SGLM\xspace}

\newcommand{\adapterfull}{Semantic-aware Teacher Refining\xspace}

\newcommand{\adapter}{SATR\xspace}

\newcommand{\plfilterfull}{Reliable Label Selection\xspace}

\newcommand{\plfilter}{RLS\xspace}

\newcommand{\dstllossfull}{Semantic Predictive Distillation\xspace}
\newcommand{\dstllossfullb}{\textbf{S}emantic \textbf{P}redictive \textbf{D}istillation\xspace}
\newcommand{\dstlloss}{SPD\xspace}

\newcommand\best[1]{\textbf{#1}}

\newcommand{\kitti}{KITTI\xspace}
\newcommand{\nuscenes}{nuScenes\xspace}

\newcommand{\leadingtriangle}{$\blacktriangleright$\xspace}
\newcommand{\leadingtriangledown}{$\blacktriangledown$\xspace}

\DeclareMathOperator{\sg}{sg}

\begin{document}

\title{Unsupervised Point Cloud Registration via Training-Time Semantic Guidance}

\titlerunning{\mymethodvivid}

\author{Kezheng Xiong\inst{1,2,*}\orcidlink{https://orcid.org/0009-0007-9368-8201} \and
Shiyun Xu\inst{1,2,*}\orcidlink{https://orcid.org/0009-0007-7224-2163} \and
Sheng Ao\inst{1,2}\orcidlink{https://orcid.org/0000-0001-6896-1869} \and
Siqi Shen\inst{1,2}\orcidlink{https://orcid.org/0000-0001-7479-0189} \and
Cheng Wang\inst{1,2}\orcidlink{https://orcid.org/0000-0001-6075-796X} \and
Chenglu Wen\inst{1,2,\textrm{\Letter}}\orcidlink{https://orcid.org/0000-0002-6189-1236}}

\authorrunning{K.~Xiong, S.~Xu et al.}

\institute{Key Laboratory of Multimedia Trusted Perception and Efficient Computing, \\
Ministry of Education of China, Xiamen University, China \and
Fujian Key Laboratory of Urban Intelligent Sensing and Computing, \\
Xiamen University, China\\
\email{\{xiongkezheng, shiyunxu\}@stu.xmu.edu.cn},\\
\email{\{aosh, siqishen, cwang, clwen\}@xmu.edu.cn}
}

\maketitle
\makeatletter
\def\blfootnote{\xdef\@thefnmark{}\@footnotetext}
\makeatother
\blfootnote{\textsuperscript{\textrm{\Letter}} Corresponding author.}
\blfootnote{$^{*}$ Equal contribution.}
\begin{center}
\includegraphics[width=0.98\textwidth]{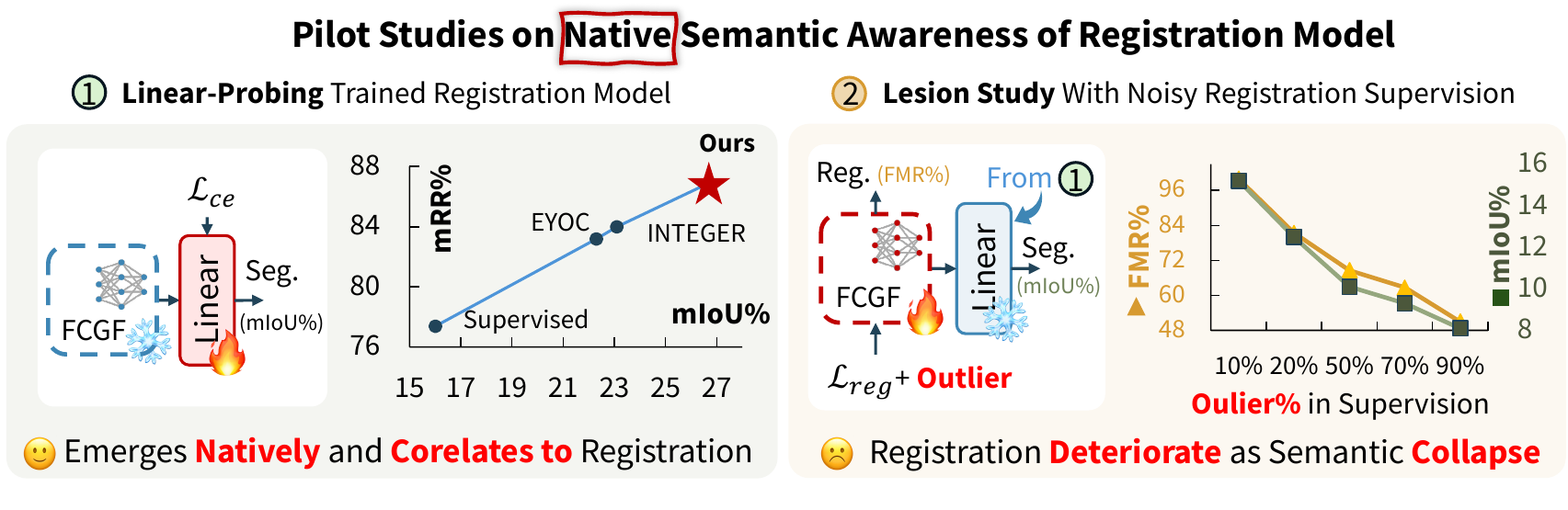}
\captionof{figure}{
    \textbf{Motivation}: Through linear probing on SemanticKITTI\cite{behley2019semantickitti}, we find that a \emph{fragile yet potent} semantic awareness \emph{emerges natively} in registration networks even without explicit semantic supervision, and it correlates with registration accuracy. 
}
\label{fig:motivation}
\end{center}

\begin{abstract}
Unsupervised registration of large-scale LiDAR point clouds remains challenging due to the geometric ambiguity inherent in outdoor scenes, which degrades pseudo-label quality and leads to suboptimal convergence, particularly for sparse, low-resolution scans such as those from \nuscenes.
We reveal that registration models \textbf{intrinsically} encode semantic awareness that strongly correlates with registration accuracy, albeit without explicit semantic supervision. However, this native awareness is fragile: noisy supervision arising from geometric ambiguity in unsupervised settings rapidly erodes the learned semantic structure, causing performance collapse.
To this end, we propose \mymethodvivid, a teacher-student framework guided by an off-the-shelf 3D segmentation model \emph{exclusively during training}. 
We observe that potential inlier matches are often buried just beneath a few spurious neighbors in the noisy feature space, motivating \matcherfull to recover them through reselection rather than simply rejecting.
Building on this, a train-only \labelerfull performs lightweight, batch-specific teacher refinement and mines reliable pseudo-labels under semantic guidance. We further introduce \dstllossfull to consolidate the student's semantic awareness in the feature space.
Extensive experiments on \kitti{} and \nuscenes{} demonstrate state-of-the-art performance, with pronounced gains on the challenging \nuscenes benchmark.
Crucially, \mymethod incurs \textbf{zero inference overhead} and requires \textbf{no semantic annotations} on the registration data.
Code will be released.
\end{abstract}
    
\section{Introduction}
\label{sec:introduction}
\begin{wrapfigure}{r}{0.52\textwidth}
	\centering
	\includegraphics[width=\linewidth]{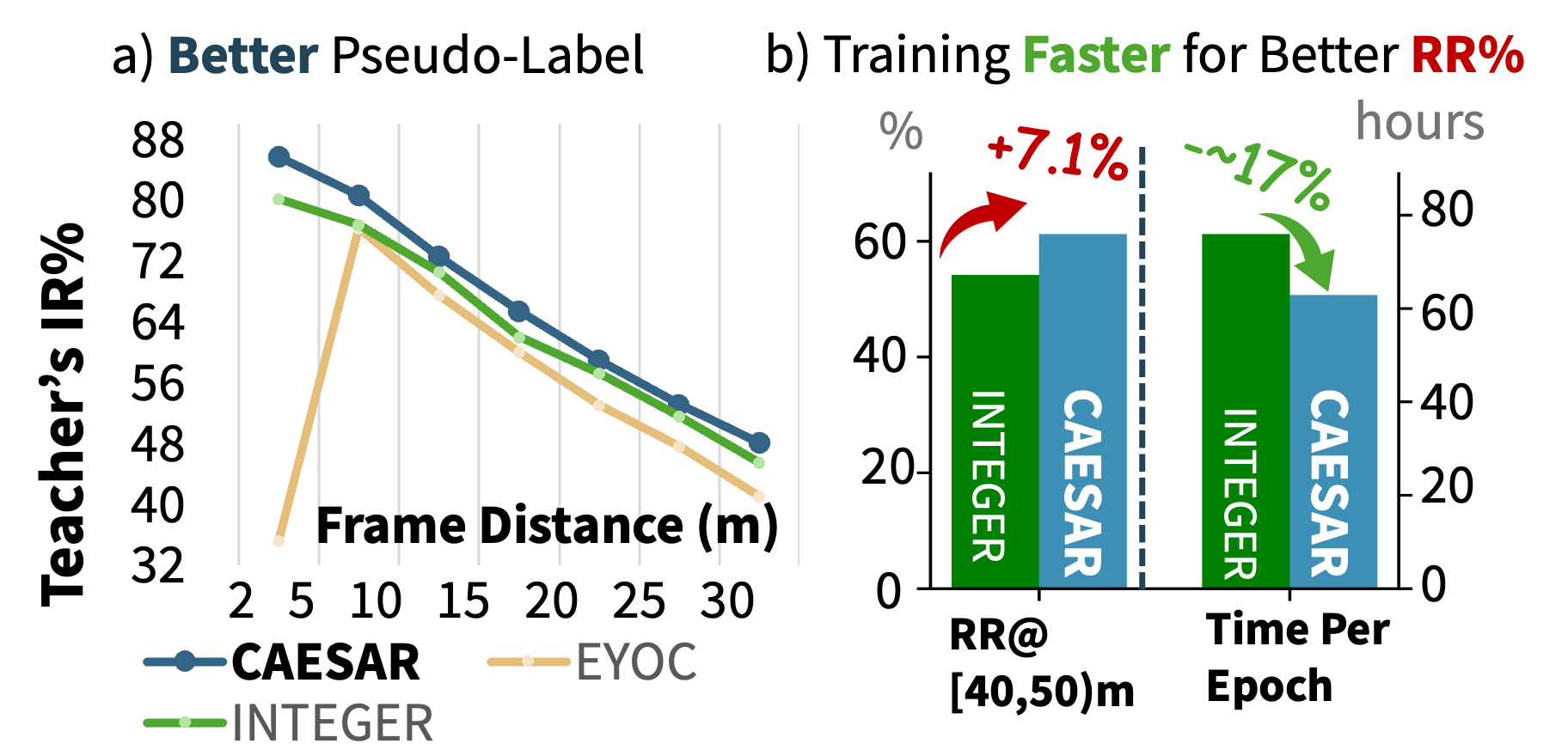}
	\caption{\textbf{Performance:} Our method leverages this native ability to improve unsupervised point cloud registration even with reduced training time (\cref{sec:analyze}).
	}
	\label{fig:performance}
\end{wrapfigure}

Point cloud registration is fundamental to 3D perception in robotics and autonomous driving, with broad applications in downstream tasks such as SLAM and scene reconstruction~\cite{zhu2022nice,jiang2023se}. While learning-based approaches achieve impressive accuracy, they require large, meticulously annotated datasets that are costly to obtain. Unsupervised methods~\cite{shen2022reliable,liu2024extend,xiong2024mining} offer compelling alternatives by generating correspondence pseudo-labels from internal consistency cues, thereby eliminating the need for pose annotations.

However, the geometric ambiguity inherent in outdoor scenes (\eg, repetitive building facades, featureless roads) degrades pseudo-label reliability in existing frameworks. Self-distillation methods~\cite{liu2024extend}, even those that mine latent contextual cues~\cite{xiong2024mining}, suffer from a \emph{self-referential loop}. Confirmation bias from noisy pseudo-labels gradually corrupts the teacher, which in turn degrades student supervision. This loop is particularly damaging for sparse, low-resolution LiDAR data (\eg, \nuscenes), where geometric cues alone are insufficient. Since internally derived signals fail to prevent collapse on their own, stable \emph{external guidance} is needed to break this loop and preserve feature-space integrity.

Semantic information is a natural candidate for such guidance, having long been recognized for resolving geometric ambiguity~\cite{yu2023peal,fung2024semreg,yan2024ml}. However, existing approaches predominantly operate in fully supervised settings~\cite{yu2023peal,fung2024semreg,yan2024ml}, leaving the unsupervised regime largely unexplored. Moreover, they exhibit notable practical limitations. Their dependence on \textbf{\emph{inference-time semantic segmentation models}} incurs substantial computational overhead and hinders generalization to domains where such models may be unavailable~\cite{xu2025s2reg,fung2024semreg}. Reliance on hard segmentation predictions for outlier rejection~\cite{yan2024ml} further induces false negatives and thus discard valuable correspondences, because they are \textbf{\emph{notoriously sensitive to class-boundary noise}} and deteriorate in out-of-domain scenes.

Beyond these practical limitations, we argue that existing approaches~\cite{xu2025s2reg,fung2024semreg,yan2024ml} conceptually mischaracterize the role of semantics, treating it as an auxiliary side-channel that remains permanently coupled to the test-time pipeline. Yet, semantic awareness is arguably an \emph{intrinsic} property of the registration model itself. To establish reliable correspondences, a network must implicitly learn to recognize coherent structures such as buildings or vegetation.
If this hypothesis holds, it unlocks a highly efficient paradigm, in which we could effectively preserve the network's innate semantic structure during training, thereby achieving superior performance with \emph{strictly zero} inference overhead.

Driven by this potential, we conduct pilot studies (see \cref{fig:motivation}). Through linear probing, we confirm that registration networks indeed exhibit \emph{native semantic awareness} that correlates well with \emph{registration} performance.
However, the lesion analysis reveals a critical fragility. Even moderate noise in supervision disrupts both geometric alignment and the latent semantic structure. We term this failure mode \emph{semantic collapse}, a phenomenon that is particularly devastating under the inherently noisy supervision of unsupervised registration.

To this end, we introduce \mymethodvivid to \textbf{C}rystallize and consolid\textbf{A}te such nativ\textbf{E} \textbf{S}emantic \textbf{A}wa\textbf{R}eness, a teacher-student framework that leverages a frozen, pretrained 3D segmentation model \emph{exclusively during training} to provide stable semantic anchors.
As a result, it prevents semantic collapse without incurring any inference cost and even \emph{reduces} training time (\cref{fig:performance}).
To mine reliable pseudo-labels with the teacher,
we observe that correct correspondences are often buried just beneath a few spurious feature-space neighbors, motivating \matcherfullb (\matcher) to recover them via reselection from candidate sets. 
Building on this, we propose the \emph{train-only} \labelerfullb (\labeler), which efficiently refines the teacher's feature space specifically for each mini-batch by optimizing only a compact parameter set, and then mines reliable pseudo-labels under semantic guidance.
To train the student, we introduce \dstllossfullb (\dstlloss), which consolidates the student's semantic awareness by encouraging semantic understanding in the feature space.

Extensive experiments on \kitti and \nuscenes demonstrate that \mymethod achieves state-of-the-art results, with especially notable improvements on the challenging \nuscenes benchmark. 
Crucially, the framework incurs \emph{\textbf{zero inference overhead}}, requiring no additional parameters or computation at test time. Our contributions are summarized as follows:
\begin{itemize}
\item We identify \emph{semantic collapse} as a key failure mode in unsupervised registration and are the first to propose preventing it via semantic anchoring using a frozen segmentation model during training.
\item We present \mymethod, introducing \matcher for recovering correct correspondences via reselection, \labeler for efficient batch-specific teacher refinement and reliable pseudo-label mining, and \dstlloss for consolidating the student's semantic awareness.
\item We achieve new state-of-the-art results on \kitti and \nuscenes with zero extra inference cost and reduced training time, while demonstrating robustness across diverse segmentation model choices.
\end{itemize}
\section{Related Work}

\subsubsection{Point Cloud Registration}
Learning-based methods typically follow two paradigms. 
\emph{Direct regression}\cite{xu2021omnet,aoki2019pointnetlk,huang2020feature} from global descriptors is efficient but can degrade under low overlap, 
while \emph{correspondence-based} methods\cite{choy2019fully,deng2018ppf,gojcic2019perfect,qin2022geometric,yew2022regtr,yu2023peal,xiong2024speal} establish point-level correspondences before pose estimation, offering greater robustness.
Robust estimation under heavy outlier ratios is another critical direction. 
Learned estimators\cite{bai2021pointdsc,choy2020deep,lee2021deep} surpass RANSAC\cite{fischler1981random} but require supervision; 
recent non-parametric methods exploit spatial compatibility\cite{chen2022sc2} or graph structures\cite{zhang20233d,zhang2024fastmac} for efficient, training-free inlier selection.

\subsubsection{Unsupervised Registration}
The high cost of pose annotations motivates unsupervised methods. Early attempts use Gaussian mixture models\cite{mei2023unsupervised,huang2022unsupervised}, neighborhood consensus\cite{shen2022reliable}, or ICP-based reward proxies\cite{jiang2021sampling}, yet often struggle with the sparsity and low overlap characteristic of outdoor LiDAR data. Teacher-student frameworks\cite{yang2021sgp} reduce pseudo-label noise via spatial filtering\cite{liu2024extend} and are further improved by contextual mining\cite{xiong2024mining}, but still rely solely on internal cues from the registration model, causing quality degradation when geometric ambiguity intensifies. We address this by harnessing the model's intrinsic semantic awareness through external semantic anchoring.

\subsubsection{Semantic Information in Registration}
Prior works leverage semantic cues for correspondence matching~\cite{yan2024ml}, outlier rejection~\cite{zhao2024sgor}, or dynamic-object suppression~\cite{cui2024sage}. Outlier rejection\cite{zhao2024sgor} using discrete, hard class predictions is brittle at category boundaries. Feature-fusion approaches~\cite{xu2025s2reg,fung2024semreg} combine semantic and geometric descriptors but require frozen segmentation models at inference or 2D--3D projection pipelines, adding considerable overhead. In contrast, we show that registration models \emph{intrinsically} develop semantic awareness, a capability that is potent yet fragile under noisy pseudo-labels. We thus propose to stabilize and strengthen this inherent awareness using external guidance \emph{only during training}, eliminating inference-time cost while achieving superior performance.

\begin{figure*}[t]
    \centering
    \includegraphics[width=0.98\linewidth]{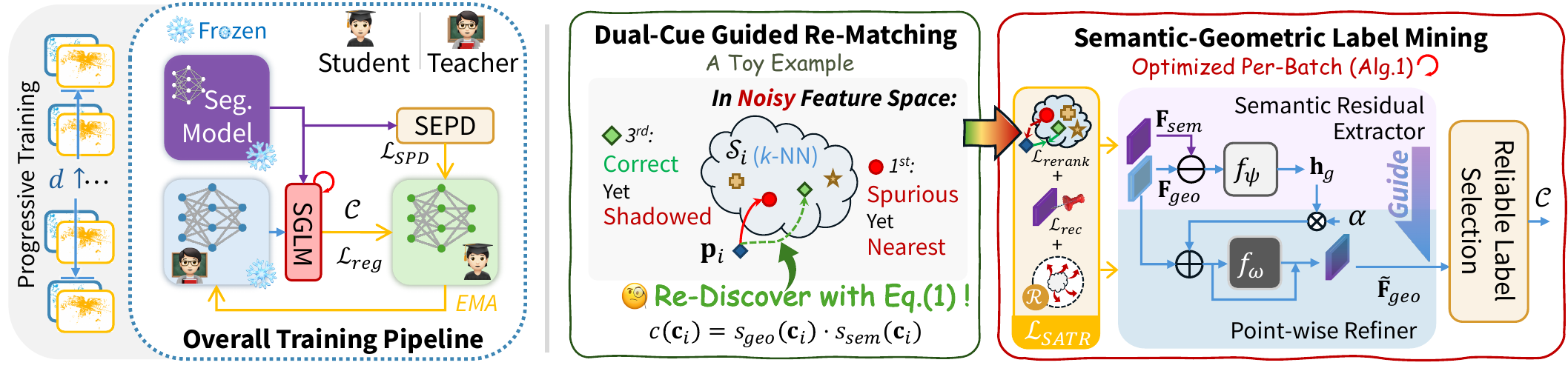}
    \caption{\textbf{Overall Pipeline.} 
    \matcher recovers correct correspondences obscured by feature-space noise by jointly considering geometric saliency and semantic similarity.
    The \emph{train-only} \labeler refines the teacher's feature space and mines high-quality pseudo-labels. 
    The student is trained with its original registration loss on \(\mathcal{C}\) and the \dstlloss objective to strengthen semantic awareness.
    }
    \label{fig:pipeline}
\end{figure*}

\section{Method}

\paragraph{Problem Formulation}
Given two LiDAR point clouds \(\mathcal{P} = \{\mathbf{p}_i\} \in \mathbb{R}^{m \times 3}\) and \(\mathcal{Q} = \{\mathbf{q}_j\} \in \mathbb{R}^{n \times 3}\), point cloud registration aims to estimate the rigid transformation \(\mathbf{T} = \{\mathbf{R}, \mathbf{t}\}\) with \(\mathbf{R} \in \mathrm{SO}(3)\) and \(\mathbf{t} \in \mathbb{R}^{3}\) that aligns \(\mathcal{P}\) to \(\mathcal{Q}\).

\paragraph{Overall Pipeline}
As shown in \cref{fig:pipeline}, \mymethod semantically anchors the teacher-student framework with a frozen, pretrained 3D segmentation model \(\mathcal{F}_{\text{sem}}\)\cite{lai2023spherical,choy20194d}.
Given a point cloud pair \((\mathcal{P}, \mathcal{Q})\), the teacher extracts geometric features \((\mathbf{F}^\mathcal{P}_{\text{geo}}, \mathbf{F}^\mathcal{Q}_{\text{geo}})\), while \(\mathcal{F}_{\text{sem}}\) provides semantic embeddings \((\mathbf{F}^\mathcal{P}_{\text{sem}}, \mathbf{F}^\mathcal{Q}_{\text{sem}})\) and per-point class logits \((\mathbf{z}^\mathcal{P}, \mathbf{z}^\mathcal{Q})\).

For the teacher, \matcher recovers correct correspondences obscured by a few spurious nearest neighbors in the noisy feature space, jointly leveraging geometric saliency and semantic consistency.
Building on this, \labeler performs lightweight, batch-specific teacher refinement, optimizing only a compact parameter set to yield semantically anchored features \((\tilde{\mathbf{F}}^\mathcal{P}_{\text{geo}}, \tilde{\mathbf{F}}^\mathcal{Q}_{\text{geo}})\). It then mines high-fidelity pseudo-labels \(\mathcal{C}\) by jointly enforcing geometric and semantic consistency.
The student model is trained with its original registration loss \(\mathcal{L}_{\text{reg}}\) (\eg, Hardest Contrastive Loss\cite{choy2019fully}) on \(\mathcal{C}\), while a novel \dstlloss strengthens the student's semantic awareness by encouraging contextual understanding rather than point-wise mimicry.
Following \cite{liu2024extend}, we adopt progressive training to gradually extend from adjacent to distant frame pairs.

Crucially, \(\mathcal{F}_{\text{sem}}\) is employed \emph{solely at training time}, incurring zero inference overhead and requiring no semantic annotations on the registration data. The framework is robust to semantic model quality, enabling training even when off-the-shelf models are imperfect or trained on other domains (see \cref{sec:analyze}).

\paragraph{Progressive Training} gradually adapts the student to increasingly distant point cloud pairs~\cite{liu2024extend,xiong2024mining}.
For each pair, the frame interval \(I\!\!\in\!\!\{1,\ldots, I_{\max}\}\) is stochastically determined, where \(I_{\max}\) increases linearly from 1 to a predefined maximum \(I_{\text{end}}\) over the course of training epochs. 
\(I\! =\!1\) represents easy adjacent frames, while larger \(I\) corresponds to low-overlap (\eg, \(\le\)40\%) and challenging pairs.

\subsection{\matcherfull}

\begin{wrapfigure}{r}{0.5\textwidth}
    \centering
    \includegraphics[width=\linewidth]{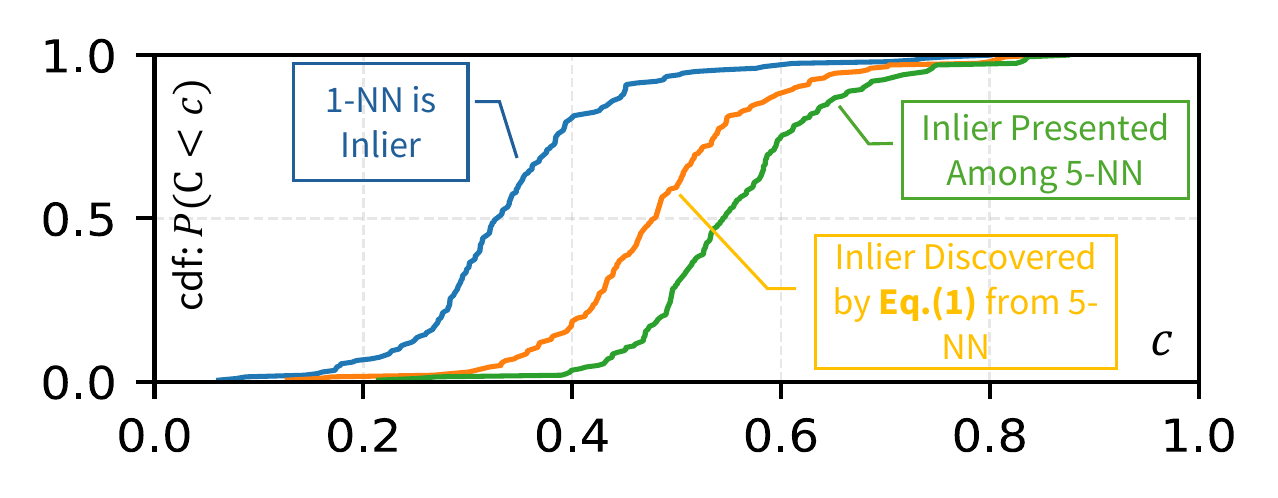}
    \caption{\textbf{Cumulative Distribution Function (cdf) of Pairs from \kitti.} 
    Variable \(C\) refers to the \emph{ratio of correct candidate sets} under certain criteria. 
    During early training, 
    over 80\% pairs have lower than 40\% correct NN matches, 
    whereas most pairs include twice as many correct matches within their top-5 candidate sets, demonstrating the potential of reselection.
    }
    \label{fig:topk_recall}
\end{wrapfigure}

\subsubsection{Motivation}
Existing methods match each point to its nearest neighbor (NN) in feature space\cite{liu2024extend,xiong2024mining,choy2019fully,huang2021predator} and then \emph{reject} suspected outliers via filtering\cite{liu2024extend} or clustering\cite{xiong2024mining}.
This paradigm faces a fundamental dilemma: when noise dominates, filtering either fails due to vast geometric ambiguity, or discards most matches, thereby ``starving''\cite{chen2022sc2,chen2023sc} the registration process that requires adequate correspondences\cite{zhao2025progressive}.

We make a \emph{key observation} (\cref{fig:topk_recall}). Although the nearest-neighbor \emph{point} may be incorrect, the true match frequently still lies within the top-\(K\) candidate \emph{set}.
This motivates a paradigm shift \textbf{\emph{from rejection to reselection}} (see the toy example in \cref{fig:pipeline}~(Mid.)). Rather than discarding noisy matches\cite{liu2024extend,mei2023unsupervised}, \matcher selects the most plausible correspondence from each candidate set, thereby recovering ample high-quality matches even under severe ambiguity, which is a crucial prerequisite for accurate registration~\cite{bai2021pointdsc,chen2022sc2}.

\subsubsection{Correspondence Scoring}
\label{sec:sgcs}
To realize this selection strategy, \matcher scores each candidate correspondence by jointly requiring both \emph{geometric saliency} and \emph{semantic similarity}, respectively denoted as \(s_{\text{geo}}\) and \(s_{\text{sem}}\):
\begin{equation}
c(\mathbf{c}_i) = s_{\text{geo}}(\mathbf{c}_i) \cdot s_{\text{sem}}(\mathbf{c}_i),
\label{eq:score}
\end{equation}
The two terms are naturally complementary, as registration models struggle with repetitive geometric structures, whereas segmentation models are uncertain near class boundaries. By requiring high scores in \emph{both} aspects simultaneously, \cref{eq:score} mitigates each source's failure mode.

Specifically, the term \(s_{\text{geo}}\) rewards structural distinctiveness\cite{pauly2002efficient,avidar2016point}. From each point's spatial neighborhood \(\mathcal{N}_k(\mathbf{p}_i)\), the covariance eigenvalues \(\lambda_1 \!\le\! \lambda_2 \!\le\! \lambda_3\) yield a saliency ratio \(\rho^{\mathbf{p}_i} \!=\! \lambda_1/(\lambda_1\!+\!\lambda_2\!+\!\lambda_3)\), which peaks at non-planar regions such as edges and corners\cite{belton2008improving}. The pairwise geometric score is then:
\begin{equation}
s_{\text{geo}}(\mathbf{c}_i) \!=\! \rho^{\mathbf{p}_i} \!\cdot\! \rho^{\mathbf{q}_i}
\end{equation}
Meanwhile, rather than solely comparing hard class assignments~\cite{yan2024ml}, which are brittle near boundaries, \(s_{\text{sem}}\) uses a weighted Jensen-Shannon divergence (JSD)~\cite{fuglede2004jensen}:
\begin{equation}
s_{\text{sem}}(\mathbf{c}_i) = w_e\exp (-w_c\cdot\mathrm{JSD}(\mathbf{z}^{\mathbf{p}_i},\,\mathbf{z}^{\mathbf{q}_j})),
\label{eq:sem_score}
\end{equation}
where \(\mathbf{z}\) denotes \(C\)-class softmax probabilities from \(\mathcal{F}_{\text{sem}}\).
The confidence weight \(w_c=(H(\mathbf{z}^{\mathbf{p}_i})+H(\mathbf{z}^{\mathbf{q}_j}))/2\log C\) penalizes uncertain predictions\cite{tarubinga2025farcluss} via the entropy \(H(\cdot)\) of the class distribution\cite{shannon1948mathematical}.  
The embedding weight \(w_e=(1+\cos(\mathbf{F}^{\mathbf{p}_i}_{\text{sem}}, \mathbf{F}^{\mathbf{q}_i}_{\text{sem}}))/2\) captures holistic semantic context via the cosine similarity of \(\mathcal{F}_{\text{sem}}\)'s embeddings, further mitigating boundary noise and providing complementary cues for disambiguation.

\subsubsection{Constrained Correspondence Reselection}
\label{sec:reselection}
For each point \(\mathbf{p}_i\), we construct a candidate set from its \(k\)-NN in geometric feature space:
\begin{equation}
\mathcal{S}_i=\{\mathbf{s}_k\mid\mathbf{s}_k=(\mathbf{p}_i,\mathfrak{N}^k_\mathcal{Q}(\mathbf{p}_i))\}_{k=1}^{K},
\end{equation}
where \(\mathfrak{N}^k_\mathcal{Q}(\mathbf{p}_i)\) is the \(k\)-th NN of \(\mathbf{p}_i\) in \(\mathcal{Q}\).
We select the highest-scoring candidate with \(\mathbf{c}_i = \arg\max_{\mathbf{s}_k \in \mathcal{S}_i} c(\mathbf{s}_k)\).
Applying similarity thresholds\cite{liu2024extend,xiong2024mining}, we obtain the refined set \(\hat{\mathcal{C}}=\{\mathbf{c}_i\}\).
Crucially, \(\hat{\mathcal{C}}\) and its candidate sets \(\hat{\mathcal{S}}=\{\mathcal{S}_i\}\) serve as a \emph{feature-space quality diagnostic}. In a discriminative feature space, correct matches should rank near the top.
We formalize this via the \emph{average rank} (AR):
\begin{equation}
\begin{aligned}
\textstyle
\text{AR}(\mathcal{C}) = \frac{1}{|\mathcal{C}|} \sum_{(\mathbf{p}_i,\mathbf{q}_j) \in \mathcal{C}} \min_k \{k \mid \mathfrak{N}^k_\mathcal{Q}(\mathbf{p}_i)= \mathbf{q}_j \}.
\end{aligned}
\label{eq:ar}
\end{equation}
A low AR indicates a discriminative feature space where correct matches surface near the top of each candidate set; a high AR signals that true correspondences remain obscured by noise. Beyond its diagnostic role, AR also motivates an effective, differentiable training objective in \labeler (\cref{sec:adapter}).

\subsection{\labelerfull}
While \matcher recovers correspondences obscured by noise, it cannot compensate once true matches fall beyond reselection's reach as semantic collapse intensifies (\cref{sec:introduction}, \cref{fig:motivation}).
The \emph{train-only} \labeler addresses this root cause by stabilizing the feature space itself, thereby \emph{preserving} the teacher's latent semantic awareness.
It proceeds in two stages. \adapterfull first reinforces the teacher's features using stable semantic anchors from \(\mathcal{F}_{\text{sem}}\) by optimizing a compact parameter set specific to each batch, and \plfilterfull then selects reliable pseudo-labels from the refined correspondences.

\subsubsection{Stage 1: \adapterfull (\adapter)}
\label{sec:adapter}
When pseudo-label noise severely accumulates, reselection alone cannot reverse the ensuing semantic erosion. 
Existing efforts either freeze the teacher entirely~\cite{liu2024extend}, forfeiting adaptability to evolving noise, or adapt the full teacher per-batch~\cite{xiong2024mining}, incurring prohibitive cost and risking overfitting. 
With \(\mathcal{F}_{\text{sem}}\) as a stable reference, we instead design a concise train-only adapter to reinforce the eroding semantic structure in the feature space, thereby preventing collapse during training.
\adapter is far more efficient, since it optimizes a mere \emph{0.07\,M} parameters (vs.\ \emph{8.76\,M} adapted in previous state-of-the-art~\cite{choy2019fully}) while achieving better performance. 
The efficiency is ensured not only by the compact design, but also through the AR metric (\cref{eq:ar}), which directly exposes noise-induced degradation and can thus guide a refinement process to restore a discriminative feature space.

\paragraph{Architecture}
Semantic embeddings reside in a far higher-dimensional space (\(d_{\text{sem}}\!=\!512\)~\cite{choy20194d} vs.\ \(d_{\text{geo}}\!=\!32\)~\cite{choy2019fully}). Na\"ive fusion or concatenation~\cite{kong2020semantic,liu2024deep,fung2024semreg,xu2025s2reg} risks diluting geometric details or conflating geometrically distinct points of the same class (\eg, separate trees/vehicles).
Therefore, \adapter is designed to extract a compact \emph{semantic residual} containing the absent, discriminative information present in \(\mathcal{F}_{\text{sem}}\).
Targeting the ambiguity revealed by \matcher (\cref{eq:ar}) within candidate sets, its parameters are freshly initialized for each batch and optimized \emph{locally} in a regularized way (\cref{eq:explore,eq:rerank}). This controlled and adaptive refinement enables \mymethod to remain effective even with imperfect semantics (see \cref{sec:analyze}).

Specifically, the \adapter comprises a \emph{residual extractor} \(\Phi_g\) and a \emph{point-wise refiner} \(\Phi_r\), connected via a learnable gate \(\alpha\).
\(\Phi_g\) projects geometric features \(\mathbf{F}^\mathcal{P}_{\text{geo}} \!\in\! \mathbb{R}^{N \times d_{\text{geo}}}\) and semantic embeddings \(\mathbf{F}^\mathcal{P}_{\text{sem}} \!\in\! \mathbb{R}^{N \times d_{\text{sem}}}\) into a shared latent space and isolates discriminative semantic signals absent from the geometric representation through a sparse convolutional block~\cite{choy2019fully} \(f_\psi\):
\begin{equation}
\mathbf{h}^{\mathcal{P}}_g =f_\psi(\mathbf{F}^\mathcal{P}_{\text{sem}}\mathbf{W}_{\text{sem}} - \mathbf{F}^\mathcal{P}_{\text{geo}}\mathbf{W}_{\text{geo}})\mathbf{W}_{\text{out}}.
\label{eq:high}
\end{equation}
To keep these residuals faithful to the semantic source, \(\Phi_g\) is trained to recover the original semantic embeddings via an auxiliary projection \(f_\phi\):
\begin{equation}
\mathcal{L}^\mathcal{P}_{\text{rec}}=\|\mathbf{F}^\mathcal{P}_{\text{sem}}-f_\phi(\mathbf{F}^{\mathcal{P}}_{\text{geo}} + \sg(\tanh(\alpha)) \cdot \mathbf{h}^{\mathcal{P}}_g)\|_2^2;
\label{eq:recon_loss}
\end{equation}
Owing to the complementary stop-gradient placement (cf.\ \cref{eq:low}), \(\mathcal{L}_{\text{rec}}\) propagates gradients only through \(\Phi_g\). The same applies to \(\mathcal{Q}\), giving \(\mathcal{L}_{\text{rec}} = \mathcal{L}^\mathcal{P}_{\text{rec}} + \mathcal{L}^\mathcal{Q}_{\text{rec}}\).
Next, \(\Phi_r\) incorporates the residuals into the geometric space via the learnable gate \(\alpha\) and a lightweight MLP \(f_\omega\):
\begin{equation}
\mathbf{h}^{\mathcal{P}}_r \! =\! \mathbf{F}^{\mathcal{P}}_{\text{geo}} + \tanh(\alpha) \cdot \sg(\mathbf{h}^{\mathcal{P}}_g), \quad 
\tilde{\mathbf{F}}^{\mathcal{P}}_{\text{geo}}\! =\! \mathbf{h}^{\mathcal{P}}_r +  f_\omega(\mathbf{h}^{\mathcal{P}}_r),
\label{eq:low}
\end{equation}
where \(\sg(\cdot)\) blocks gradients from \(\Phi_r\) to \(\Phi_g\).
Together with the \(\sg(\cdot)\) placement in \cref{eq:recon_loss}, these stop-gradients realize \emph{alternating optimization}, in which each loss updates one module while implicitly treating the other as fixed, allowing \(\Phi_g\) and \(\Phi_r\) to specialize without conflicting gradients.

To sharpen feature discriminability with the incorporated semantic cues, we propose a contrastive re-ranking loss that exploits the candidate-set structure from \matcher, pulling the selected correspondence \(\mathbf{c}_i \in \hat{\mathcal{C}}\) closer while repelling the remaining \(K\!-\!1\) candidates~\cite{oord2018representation}.
This is achieved via a \emph{differentiable surrogate} of the AR diagnostic that smoothly relaxes \cref{eq:ar}:
\begin{equation}
    \label{eq:rerank}
    \textstyle
\mathcal{L}_{\text{rerank}} = \frac{1}{|\hat{\mathcal{C}}|} \sum_{\mathbf{c}_i \in \hat{\mathcal{C}}} \log \left( 1 + \sum_{\mathbf{q}_k \in \mathcal{S}_i \setminus \{\mathbf{q}_j\}} \exp \left( \frac{s(\mathbf{p}_i, \mathbf{q}_k) - s(\mathbf{p}_i, \mathbf{q}_j)}{\tau} \right) \right)
\end{equation}
Please refer to the supplementary material for full mathematical proof of this relaxation.
Here, \(s(\cdot, \cdot)\) denotes cosine similarity between refined features and \(\tau\) is a temperature. 
The same computation applies to \(\mathcal{Q}\).
Unlike standard contrastive registration losses\cite{choy2019fully,huang2021predator,qin2022geometric,zeng2025unlocking}, \(\mathcal{L}_{\text{rerank}}\) operates \emph{locally} within each \(\mathcal{S}_i\), confining the influence of erroneous selections to their immediate neighborhoods.

\paragraph{Volumetric Regularized Objective}
While \(\mathcal{L}_{\text{rerank}}\) and \(\mathcal{L}_{\text{rec}}\) ensure each component serves its intended role, one concern remains.
Incorporating high-dimensional semantic residuals into a compact \(d_{\text{geo}}\)-dimensional space risks diluting its volume\cite{yu2020learning}, neutralizing the geometric discriminative capacity upon which \mymethod relies for pseudo-labeling.
We prevent this with an exploration regularizer that maximizes the geometric volume of the feature distribution~\cite{yu2020learning,wu2025simplifying}:
\begin{equation}
\mathcal{R}^\mathcal{P} = -\gamma \frac{1}{2}\log\det(\mathbf{I} + \frac{1}{\epsilon^2}\Gamma^\mathcal{P}_{\text{batch}}),
\label{eq:explore}
\end{equation}
where \(\Gamma^\mathcal{P}_{\text{batch}} \in \mathbb{R}^{d \times d}\) is the batch-wise covariance matrix and \(\gamma, \epsilon\) control regularization strength.
Aggregating all terms yields the full optimization objective:
\begin{equation}
\label{eq:full_loss}
\mathcal{L}_{\text{\adapter}}=\underbrace{\mathcal{R}^\mathcal{P} + \mathcal{R}^\mathcal{Q}}_{\text{Exploration}} + \underbrace{\mathcal{L}_{\text{rerank}} + \mathcal{L}_{\text{rec}}}_{\text{Exploitation}}.
\end{equation}
This defines an effective exploration--exploitation loop. Volume maximization creates room for semantic residuals, re-ranking (\(\mathcal{L}_{\text{rerank}}\)) thereby sharpens local discrimination, and reconstruction \(\mathcal{L}_{\text{rec}}\) prevents the semantic guidance from drifting, keeping the adapted space both expressive and semantically grounded.

\paragraph{Batch-Specific Optimization}
With the objective in \cref{eq:full_loss}, the optimization of \adapter is driven by \cref{alg:satr}, freshly initialized for every mini-batch and optimized in an Expectation-Maximization (EM) procedure\cite{dempster1977maximum}.
We apply zero-initialization to \(f_\omega\)'s final layer and \(\alpha\), ensuring gradual, stable refinement from the original features.
In the E-step, \(\Phi_r\) and \(\Phi_g\) are frozen while \matcher mines correspondences \(\hat{\mathcal{C}}\) and their candidate sets \(\hat{\mathcal{S}}\) from the current refined features.
In the M-step, these quantities are held fixed and \(\Phi_r\), \(\Phi_g\) are updated by minimizing \(\mathcal{L}_{\text{\adapter}}\), with the stop-gradient decoupling (\cref{eq:recon_loss,eq:low}) implicitly alternating optimization between the two components, yielding a two-level optimization hierarchy.
The E/M alternation terminates once AR (\cref{eq:ar}) stabilizes, which serves as a natural convergence criterion, since AR is the metric to reflect feature-space ambiguities that \(\mathcal{L}_{\text{rerank}}\) smoothly surrogates.
A final \matcher pass then yields the refined proposals \(\tilde{\mathcal{C}}\) for \plfilterfull.
The trained modules are discarded after each batch, preventing confirmation bias from accumulating across batches.

Thresholds and other hyperparameters are detailed in the supplementary material. Empirically, \cref{alg:satr} converges stably and rapidly, typically within 3--5 iterations, yielding improved pseudo-labels while notably reducing training time compared to full model adaptation~\cite{xiong2024mining} (\cref{sec:analyze}).

\begin{algorithm}[t]
\caption{Batch-Specific Optimization for \adapterfull}
\label{alg:satr}
\begin{algorithmic}[1]
\Require Geometric/semantic features \(\mathbf{F}^{\mathcal{P}}_{\text{geo}},\mathbf{F}^{\mathcal{Q}}_{\text{geo}}\)/\(\mathbf{F}^{\mathcal{P}}_{\text{sem}},\mathbf{F}^{\mathcal{Q}}_{\text{sem}}\); logits \(\mathbf{z}^{\mathcal{P}},\mathbf{z}^{\mathcal{Q}}\)
\Ensure Refined proposals \(\tilde{\mathcal{C}}\)
\State Initialize trainable parameters from \(\Phi_r,\Phi_g\) \hfill \(\triangleright\) \emph{freshly initialized per batch}
\State Calculate \(\tilde{\mathbf{F}}^{\mathcal{P}}_{\text{geo}},\tilde{\mathbf{F}}^{\mathcal{Q}}_{\text{geo}}\) by \cref{eq:low,eq:high} on \(\mathbf{F}^{\mathcal{P}}_{\text{geo}},\mathbf{F}^{\mathcal{Q}}_{\text{geo}},\mathbf{F}^{\mathcal{P}}_{\text{sem}},\mathbf{F}^{\mathcal{Q}}_{\text{sem}}\)
\Repeat
    \State \textbf{E-step:} Fix \(\Phi_r,\Phi_g\); let \(\hat{\mathcal{C}},\hat{\mathcal{S}} \gets \text{\matcher}(\tilde{\mathbf{F}}^{\mathcal{P}}_{\text{geo}},\tilde{\mathbf{F}}^{\mathcal{Q}}_{\text{geo}})\)
    \State \textbf{M-step:} Fix \(\hat{\mathcal{C}},\hat{\mathcal{S}}\); update \(\Phi_r,\Phi_g\) by gradient descent via \(\nabla_{\Phi_r,\Phi_g}\mathcal{L}_{\text{\adapter}}\)
    \State Calculate \(\tilde{\mathbf{F}}^{\mathcal{P}}_{\text{geo}},\tilde{\mathbf{F}}^{\mathcal{Q}}_{\text{geo}}\) by \cref{eq:low,eq:high} on \(\mathbf{F}^{\mathcal{P}}_{\text{geo}},\mathbf{F}^{\mathcal{Q}}_{\text{geo}},\mathbf{F}^{\mathcal{P}}_{\text{sem}},\mathbf{F}^{\mathcal{Q}}_{\text{sem}}\)
\Until{\(\text{AR}(\hat{\mathcal{C}})\) difference converges within \(\epsilon\)}
\State \(\tilde{\mathcal{C}} \gets \text{\matcher}(\tilde{\mathbf{F}}^{\mathcal{P}}_{\text{geo}},\tilde{\mathbf{F}}^{\mathcal{Q}}_{\text{geo}})\) \hfill \(\triangleright\) \emph{final proposals for \plfilterfull}
\end{algorithmic}
\end{algorithm}

\subsubsection{Stage 2: \plfilterfull (\plfilter)}
With refined features from the teacher, this stage produces reliable pseudo-labels \(\mathcal{C}\) for student training.
Rather than relying on a single transformation estimated from noisy initial correspondences\cite{liu2024extend,xiong2024mining}, we generate \(K\) transformation hypotheses \(\{\mathbf{T}_k\}_{k=1}^K\) from the refined proposals \(\tilde{\mathcal{C}}\) via spatial compatibility measures\cite{chen2022sc2,zhang2024fastmac} and select the best by re-weighting each hypothesis's consensus set (\ie, sets of selected inliers\cite{chen2022sc2,chen2023sc}) \(\mathcal{C}_k\) with \(c(\cdot)\) defined in \cref{eq:score}:
\begin{equation}
\textstyle
k^* = \operatornamewithlimits{arg\,max}_{k}\left(\mathrm{IC}(\mathbf{T}_k) \cdot \frac{1}{|\mathcal{C}_k|}\sum_{\mathbf{c}_i \in \mathcal{C}_k} c(\mathbf{c}_i) \right).
\label{eq:wic}
\end{equation}
Here, \(\mathrm{IC}(\cdot)\) denotes the inlier count metric commonly used to evaluate transformation hypotheses\cite{fischler1981random,chen2022sc2,chen2023sc,zhang20233d,zhang2024fastmac}.
This ensures the selected transformation is grounded in correspondences that are both geometrically salient and semantically coherent.
Using \(\mathbf{T}_{k^*}\), we follow \cite{liu2024extend,xiong2024mining} to obtain \(\mathcal{C}^{'}_{k^*}\) via NN search. The final pseudo-labels are then selected as the top-\(M\) highest-scoring correspondences:
\begin{equation}
\mathcal{C} = \{\mathbf{c}_i \mid \mathbf{c}_i \in \mathcal{C}^{'}_{k^*}, i \in \mathrm{TopM}(\{c(\mathbf{c}_i)\})\}.
\label{eq:final_sel}
\end{equation}
This final selection step prevents \emph{systematically} incorrect pseudo-labels caused by inaccurate transformations that are more prone to memorize\cite{chapelle2006,vincent2010stacked}, thereby enhancing supervision quality and ultimately improving registration performance.

\subsection{Initialization and Training Pipeline}

\subsubsection{Teacher Initialization and Update}
Unlike prior methods that require additional synthetic pre-training~\cite{xiong2024mining} or assume identity transformations between adjacent frames as ground truth~\cite{liu2024extend} for initialization, the proposed approach of semantic anchoring enables reliable pseudo-labeling on adjacent frame pairs, even from a randomly initialized teacher.
In the first stage of progressive training~\cite{liu2024extend}, we therefore train the student without \dstlloss to avoid propagating feature-level noise, synchronizing the teacher with the student after each epoch.
In subsequent stages, the teacher is frozen and updated once per epoch via exponential moving average (EMA) of the student weights.

\subsubsection{\dstllossfull for Student}
Training the student with the registration loss \(\mathcal{L}_{\text{reg}}\) (\eg, Hardest Contrastive Loss\cite{choy2019fully}) on the pseudo-labels \(\mathcal{C}\) transfers geometric correspondence knowledge.
However, \(\mathcal{L}_{\text{reg}}\) carries no explicit signal about the semantic structure captured by \(\mathcal{F}_{\text{sem}}\).
Directly regressing high-dimensional semantic embeddings is suboptimal due to the feature-space gap (\(d_{\text{sem}}\!\gg\!d_{\text{geo}}\), cf.\ \cref{sec:adapter}), which collapses feature diversity and dilutes geometric details.
We instead formulate \dstlloss as contextual prediction~\cite{assran2023self,soh2025noise}, where the student is encouraged to \emph{reconstruct} a point's semantic embedding from its geometric context, thereby internalizing semantic reasoning.
Concretely, distillation targets are restricted to high-confidence points \(\tilde{\mathbf{p}}\) whose correspondences \(\mathfrak{C} = \{(\tilde{\mathbf{p}}_i, \tilde{\mathbf{q}}_j)\}\subseteq \mathcal{C}\) pass mutual nearest-neighbor verification in the teacher's refined feature space.
For each selected point, we stochastically drop \(r\) feature channels and add Gaussian noise \(\boldsymbol{\epsilon}\!\sim\!\mathcal{N}(0,\sigma^2)\)~\cite{wang2024glace} to the student features, encouraging robustness and preventing overfitting. The objective is defined as:
\begin{equation}
\textstyle
\mathcal{L}_{\text{\dstlloss}}^\mathcal{P} = \sum_{\tilde{\mathbf{p}}_i \in \tilde{\mathbf{p}}} c_i \cdot \mathrm{SmoothL1}(f_\theta(\mathbf{F}^{\{\mathcal{P} \setminus \tilde{\mathbf{p}}\}}_{\text{stu}}+\boldsymbol{\epsilon}), \mathbf{F}^{\tilde{\mathbf{p}}_i}_{\text{sem}}).
\label{eq:mask_predict}
\end{equation}
Here, \(f_\theta\) is a convolutional\cite{choy20194d} auxiliary decoder trained alongside the student to reconstruct the corresponding semantic embedding.
\(c_i\) weight predictions by estimated quality (\cref{eq:score}); Smooth L1~\cite{girshickICCV15fastrcnn} provides robustness to spurious targets.
The loss is computed symmetrically for \(\mathcal{Q}\), giving \(\mathcal{L}_{\text{\dstlloss}}\!=\!\mathcal{L}_{\text{\dstlloss}}^\mathcal{P}\!+\!\mathcal{L}_{\text{\dstlloss}}^\mathcal{Q}\).

\paragraph{Overall Student Objective}
The student model is trained end-to-end as a weighted sum of both losses, with \(\lambda_{\text{\dstlloss}}\) balancing the two terms:
\begin{equation}
    \mathcal{L}_{\text{total}} = \mathcal{L}_{\text{reg}} +  \lambda_{\text{\dstlloss}} \, \mathcal{L}_{\text{\dstlloss}}, 
\end{equation}
\section{Experiments}

Following existing unsupervised methods for outdoor LiDAR data\cite{liu2024extend,xiong2024mining}, we evaluate \mymethod on \kitti\cite{geiger2012we} and \nuscenes\cite{caesar2020nuscenes}. 
Implementation details are provided in the supplementary material.

\subsection{Setup}

\paragraph{Metrics}
Following \cite{qin2022geometric, zeng2025unlocking, huang2021predator}, we report \emph{Relative Rotation Error} (RRE), \emph{Relative Translation Error} (RTE), and \emph{Registration Recall} (RR) with thresholds \(\mathrm{RRE}\!<\!5\degree\!\wedge\! \mathrm{RTE}\!<\!2\mathrm{m}\). 
For distant scenarios, we report \(\mathrm{RR}@\left[d_1,d_2\right)\) for distance \(d\in\left[d_1,d_2\right)\) and \emph{mean RR} (mRR)\cite{liu2024extend,xiong2024mining}, defined as the average RR across all distance intervals.
We use the teacher's \emph{Inlier Ratio} at the first teacher-student training epoch (tIR@1st Epoch) to evaluate pseudo-label quality\cite{liu2024extend,xiong2024mining}.

\paragraph{Baselines}
We compare against supervised methods, including \emph{geometry-only} FCGF\cite{choy2019fully}, Predator\cite{huang2021predator}, SpinNet\cite{ao2021spinnet}, D3Feat\cite{bai2020d3feat}, CoFiNet\cite{yu2021cofinet}, and GeoTrans.\cite{qin2022geometric}, as well as \emph{semantic-aware} baselines S$^2$Reg\cite{xu2025s2reg} and ML-SemReg\cite{yan2024ml}, both of which require the segmentation model at inference.
Following EYOC\cite{liu2024extend}, we report a
variant of FCGF denoted as FCGF+C, which is FCGF trained with progressive training.
For unsupervised methods, we compare with RIENet\cite{shen2022reliable}, EYOC\cite{liu2024extend}, and INTEGER\cite{xiong2024mining}.
Following~\cite{liu2024extend,xiong2024mining}, we use the FCGF model with RANSAC for transformation estimation unless otherwise specified, ensuring a fair comparison.

\subsection{Comparison with State-of-the-Art}

\begin{table}[tb]
  \centering
  \setlength{\tabcolsep}{2.0pt}
  \renewcommand{\arraystretch}{0.9}
  \caption{
  \textbf{Comparisons with State-of-the-Art Methods.} 
  The best \emph{unsupervised} results are highlighted in \textbf{bold}. 
  ``\checkmark'' in the column ``U'' denotes the methods are \emph{\textbf{U}nsupervised}. 
  Notably, S$^2$Reg and ML-SemReg require semantic information at inference which \textcolor{CarnationPink}{adds computational overhead}, 
  while \mymethod uses semantics only during training, matching the inference \textcolor{YellowGreen}{efficiency} of purely geometric methods.
  }
  \scriptsize{
  \begin{tabular}{lllccccccc}
    \toprule
    \multirow{2}{*}{Dataset} & \multirow{2}{*}{\tiny{\shortstack[l]{Semantic \\ Availability \\ Requirement}}} & \multirow{2}{*}{Method} & \multirow{2}{*}{U} & \multirow{2}{*}{mRR} & \multicolumn{5}{c}{RR\(@ d\in\)}  \\
    \cmidrule{6-10}
    & & & & & \(\left[5, 10\right)\) & \(\left[10, 20\right)\) & \(\left[20, 30\right)\) & \(\left[30, 40\right)\) & \(\left[40, 50\right)\) \\
    \midrule
    \multirow{16}{*}{\kitti} & \multirow{9}{*}{N/A} & FCGF & -- & 77.4 & 98.4 & 95.3 & 86.8 & 69.7 & 36.9 \\
    & & FCGF+C & -- & 84.6 & 100.0 & 97.5 & 90.1 & 79.1 & 56.3 \\
    & & Predator & --  & 87.9 & 100.0 & 98.6 & 97.1 & 80.6 & 63.1 \\
    & & SpinNet & --  & 39.1 & 99.1 & 82.5 & 13.7 & 0.0 & 0.0 \\
    & & D3Feat & --  & 83.0 & 99.8 & 98.2 & 89.6 & 76.1 & 51.3 \\
    & & CoFiNet & --  & 82.1 & 99.9 & 99.1 & 94.1 & 78.6 & 38.7 \\
    & & GeoTrans. & --  & 42.2 & 100.0 & 93.9 & 16.6 & 0.7 & 0.0 \\
    \cmidrule{2-10}
    & \multirow{2}{*}{\shortstack[l]{\textcolor{CarnationPink}{Inference-} \\ \textcolor{CarnationPink}{Time}}} & \cellcolor{CarnationPink!10} S$^2$Reg &  \cellcolor{CarnationPink!10} --  &  \cellcolor{CarnationPink!10} 52.2 &  \cellcolor{CarnationPink!10} 100.0 &  \cellcolor{CarnationPink!10} 94.2 &  \cellcolor{CarnationPink!10} 61.3 &  \cellcolor{CarnationPink!10} 5.0 &  \cellcolor{CarnationPink!10} 1.5 \\
    & &  \cellcolor{CarnationPink!10} ML-SemReg & \cellcolor{CarnationPink!10} --  &  \cellcolor{CarnationPink!10} 73.8 &  \cellcolor{CarnationPink!10} 98.1 &  \cellcolor{CarnationPink!10} 96.3 &  \cellcolor{CarnationPink!10} 82.0 &  \cellcolor{CarnationPink!10} 68.1 &  \cellcolor{CarnationPink!10} 24.3 \\
    \cmidrule{2-10}
    & \multirow{3}{*}{N/A} & EYOC & \checkmark & 83.2 & 99.5 & 96.6& 89.1 & 78.6 & 52.3 \\
    & & RIENet & \checkmark & 50.7 & 96.3 & 72.1 & 38.2 & 24.4 & 22.6 \\
    & & INTEGER & \checkmark & 84.0 & 99.5 & 97.1 & 89.6 & 79.6 & 54.2 \\
    \cmidrule{2-10}
    & \textcolor{YellowGreen}{\textbf{Train-Only}} & \cellcolor{YellowGreen!10} \mymethod & \cellcolor{YellowGreen!10} \checkmark & \cellcolor{YellowGreen!10} \best{86.5} & \cellcolor{YellowGreen!10} \best{99.5} & \cellcolor{YellowGreen!10} \best{97.6} & \cellcolor{YellowGreen!10} \best{91.1} & \cellcolor{YellowGreen!10} \best{83.1} & \cellcolor{YellowGreen!10} \best{61.3} \\
    \midrule

    \multirow{12}{*}{\nuscenes} & \multirow{3}{*}{N/A} & FCGF & --  & 39.5 & 87.9 & 63.9 & 23.6 & 11.8 & 10.2  \\
    & & FCGF+C & --  & 59.3 & 96.2 & 85.1 & 59.6 & 35.8 & 20.0 \\
    & & Predator & --  & 51.0 & 99.7 & 72.2 & 52.8 & 16.2 & 14.3 \\
    \cmidrule{2-10}
    & \multirow{2}{*}{\shortstack[l]{\textcolor{CarnationPink}{Inference-} \\ \textcolor{CarnationPink}{Time}}} & \cellcolor{CarnationPink!10} S$^2$Reg & \cellcolor{CarnationPink!10} -- & \cellcolor{CarnationPink!10} 33.5  & \cellcolor{CarnationPink!10} 94.0 & \cellcolor{CarnationPink!10} 50.3 & \cellcolor{CarnationPink!10} 14.0 & \cellcolor{CarnationPink!10} 9.1 & \cellcolor{CarnationPink!10} 0.0\\
    & &  \cellcolor{CarnationPink!10} ML-SemReg & \cellcolor{CarnationPink!10} -- & \cellcolor{CarnationPink!10} 43.5 & \cellcolor{CarnationPink!10} 87.0 & \cellcolor{CarnationPink!10} 77.1 & \cellcolor{CarnationPink!10} 37.5 & \cellcolor{CarnationPink!10} 12.6 & \cellcolor{CarnationPink!10} 3.2\\
    \cmidrule{2-10}
    & \multirow{3}{*}{N/A} & EYOC & \checkmark & 61.7 & 96.7 & 85.6 & 61.8 & 37.5 & 26.9 \\
    & & RIENet & \checkmark & 47.1 & 96.5 & 57.9 & 36.6 & 25.8 & 18.9 \\
    & & INTEGER & \checkmark & 63.1 & 97.1 & 86.9 & 62.9 & 39.6 & 29.4 \\
    \cmidrule{2-10}
    & \textcolor{YellowGreen}{\textbf{Train-Only}} & \cellcolor{YellowGreen!10} \mymethod & \cellcolor{YellowGreen!10} \checkmark & \cellcolor{YellowGreen!10} \best{79.5} & \cellcolor{YellowGreen!10} \best{100.0} & \cellcolor{YellowGreen!10} \best{96.4} & \cellcolor{YellowGreen!10} \best{85.6} & \cellcolor{YellowGreen!10} \best{72.2} & \cellcolor{YellowGreen!10} \best{43.1} \\
    \midrule

    \multirow{5}{*}{\shortstack[c]{\kitti \\ \textdownarrow \\ \nuscenes }} &
    \multirow{3}{*}{N/A} &
    EYOC & \checkmark & 55.3 & 96.2 & 75.6 & 58.7 & 26.6 & 19.7\\
    & & RIENet & \checkmark & 46.2 & 83.3 & 73.2 & 43.5 & 19.8 & 11.1 \\
    & & INTEGER & \checkmark & 62.6 & 97.5 & 84.6 & 62.6  & 37.8 & 30.2 \\
    \cmidrule{2-10}
    & \textcolor{YellowGreen}{\textbf{Train-Only}} & \cellcolor{YellowGreen!10} \mymethod & \cellcolor{YellowGreen!10} \checkmark& \cellcolor{YellowGreen!10} \best{74.1} & \cellcolor{YellowGreen!10} \best{99.5} & \cellcolor{YellowGreen!10} \best{94.6} & \cellcolor{YellowGreen!10} \best{82.3} & \cellcolor{YellowGreen!10} \best{56.3} & \cellcolor{YellowGreen!10} \best{37.1}  \\
    \bottomrule
  \end{tabular}}
  \label{tbl:overallperf}
\end{table}

\paragraph{Overall Performance}
\cref{tbl:overallperf} shows that \mymethod substantially outperforms all unsupervised methods, achieving state-of-the-art results.
Notably, \mymethod yields pronounced improvements on \nuscenes, validating that preventing \emph{semantic collapse} is critical for sparse, low-resolution LiDAR data.
We employ cross-dataset off-the-shelf \(\mathcal{F}_\text{sem}\) (SphereFormer\cite{lai2023spherical} trained on \nuscenes\cite{caesar2020nuscenes} for \kitti and vice versa) by default to avoid information leakage and to benchmark robustness.
Consistent gains in distant scenarios (\eg, \(d\in[20,30)\), \([30,40)\)) confirm that semantic anchoring improves discriminability when geometry becomes ambiguous.
Crucially, unlike methods\cite{xu2025s2reg,yan2024ml} that require semantic models at inference, \(\mathcal{F}_{\text{sem}}\) is employed \emph{only at training time}, matching the inference cost of purely geometric methods. 

\paragraph{Generalizability}
When trained on \kitti and tested on \nuscenes, \mymethod substantially outperforms existing methods (\(+11.5\%\) mRR vs.\ INTEGER), with marked long-range improvements (\eg, \(+6.9\%\) in \(\mathrm{RR}@[40,50)\)).
This superior cross-dataset transferability, despite significant domain gaps (64- vs.\ 32-beam LiDAR and distinct scene characteristics), suggests that \mymethod learns more generalizable features, likely because stable semantic guidance prevents overfitting to domain-specific noise. 

\begin{figure}[tb]
\begin{minipage}[l]{0.49\linewidth}
  \centering
  \setlength{\tabcolsep}{0.9pt}
  \renewcommand{\arraystretch}{0.9}
  \captionof{table}{\textbf{Impact of Segmentation Models.} 
  We train on \kitti with different \(\mathcal{F}_{\text{sem}}\).
  \emph{In-domain} refers to \(\mathcal{F}_{\text{sem}}\) trained on SemanticKITTI\cite{behley2019semantickitti} , while \emph{out-of-domain} is those trained on nuScenes\cite{caesar2020nuscenes}.
  Sim-to-real \(\mathcal{F}_{\text{sem}}\) uses no real-world annotation for semantics, showing \mymethod's practicality.
  }
  \scriptsize{
    \begin{tabular}{llccccc}
    \toprule
    \multirow{2}{*}{\#} & \multirow{2}{*}{Seg. Model} & \multirow{2}{*}{Paradigm}  & \multirow{2}{*}{mRR} & \multicolumn{3}{c}{\(d\in [40, 50)\)} \\
    \cmidrule{5-7}
    & & & & RR & RRE & RTE  \\
    \midrule
    \multicolumn{7}{l}{\leadingtriangle\textit{In-domain Supervised \(\mathcal{F}_{\text{sem}}\)}} \\
    a) & SphereFormer & \tiny{Transformer} & 87.3 & 63.3 & 0.97 & 0.27 \\
    \midrule
    \multicolumn{7}{l}{\leadingtriangle\textit{Out-of-domain Supervised \(\mathcal{F}_{\text{sem}}\)}} \\
    b) & SphereFormer & \tiny{Transformer} & 86.5 & 61.3 & 1.10 & 0.28 \\
    c) & SPVNAS & \tiny{Point-Voxel} & 85.6 & 60.8 & 1.17 & 0.30 \\
    d) & MinkUNet & \tiny{Voxel-Based}  & 85.8 & 61.3 & 1.07 & 0.29 \\
    \midrule
    \multicolumn{7}{l}{\leadingtriangle\textit{Unsupervised sim-to-real adapted \(\mathcal{F}_{\text{sem}}\)}} \\
    e) & DGT-ST & -- & 85.4 & 58.3 & 1.39 & 0.44 \\
    \bottomrule
  \end{tabular}
  }
  \label{tbl:scalessm}
\end{minipage}
\hfill
\begin{minipage}[r]{0.48\linewidth}
  \centering
  \setlength{\tabcolsep}{1.4pt}
  \renewcommand{\arraystretch}{0.8}
  \captionof{table}{Ablation Study on \kitti.}
  \scriptsize{
  \begin{tabular}{lp{1cm}cccc}
    \toprule
    \multirow{2}{*}{Variants} & {tIR@1\textsuperscript{st}} & \multirow{2}{*}{mRR} & \multicolumn{3}{c}{\(d\in\left[40,50\right)\)} \\ 
    \cmidrule{4-6}
    & Epoch & & RR & RRE & RTE \\
    \midrule
    \leadingtriangle w/o \matcher  & 83.7 & 85.1 & 59.8 & 1.34 & 0.55 \\
    \rowcolor{darkgray!20}\multicolumn{6}{l}{\leadingtriangledown \textit{reselection w/ score variants:}} \\
    \rowcolor{gray!10}\quad w/ \(s_{\text{sem}}\) & 84.8 & 85.8 & 59.3 & 1.47 & 0.38 \\
    \rowcolor{gray!10}\quad w/ \(s_{\text{geo}}\) & 84.5 & 85.6 & 59.8 &  1.12 & 0.31 \\
    \midrule
    \leadingtriangle w/o \adapter & 83.6 & 84.8 & 59.3 & 1.23 & 0.51 \\
    \rowcolor{darkgray!20}\multicolumn{6}{l}{\textit{\leadingtriangledown w/ \adapter variants:}}\\
    \rowcolor{gray!10}\quad w/o \(\Phi_g\) & 84.9 & 85.7 & 59.8 & 1.30 & 0.41 \\
    \rowcolor{gray!10}\quad w/o \(\mathcal{R}\) & 84.7 & 85.8 & 60.3 & 1.41 & 0.40 \\
    \rowcolor{gray!10}\quad \(\mathcal{L}_\text{rerank}\) \textrightarrow \(\mathcal{L}_\text{reg}\)  & 83.3 & 85.0 & 59.8 & 1.25 & 0.49 \\
    \midrule
    \leadingtriangle w/o \plfilter  & 81.2 & 85.2 & 60.3 & 1.62 & 0.53  \\
    \rowcolor{darkgray!20}\multicolumn{6}{l}{\leadingtriangledown \textit{w/ \plfilter variants:}} \\
    \rowcolor{gray!10}\quad \cref{eq:wic} \textrightarrow \(\mathrm{IC}\)  & 82.0  & 85.7 & 59.8 & 1.21 & 0.44 \\
    \rowcolor{gray!10}\quad w/o \cref{eq:final_sel} & 83.2 & 85.5 & 59.3 & 1.46 & 0.52 \\
    \midrule
    \leadingtriangle w/o \dstlloss & 85.6 & 85.0 & 58.8 & 1.26 & 0.41 \\
    \leadingtriangle \(L_\text{\dstlloss}\) \textrightarrow \(L_1\) & 85.6 & 84.0 & 57.8 & 1.51 & 0.49 \\
    \midrule
    \rowcolor{blue!8} Full \mymethod & 85.6 & 86.5 & 61.3 & 1.10 & 0.28 \\
    \bottomrule
  \end{tabular}}
  \label{tbl:ablation}
\end{minipage}
\end{figure}

\subsection{Analysis}
\label{sec:analyze}

\paragraph{Impact of Segmentation Model Choice and Quality}
\cref{tbl:scalessm} shows that \mymethod maintains robust performance across diverse \(\mathcal{F}_{\text{sem}}\) choices, including voxel-based\cite{choy20194d}, point-voxel hybrid\cite{tang2020searching}, transformer-based\cite{lai2023spherical}, and \emph{unsupervised} sim-to-real domain-adapted\cite{yuan2024density} models.
Performance is stable with \textbf{\emph{both in-domain and out-of-domain semantics}}, demonstrating robustness to domain shift in semantic priors.
Crucially, the unsupervised DGT-ST's semantic quality \textbf{\emph{is significantly inferior}} to that of supervised methods (SemanticKITTI mIoU: 43.1\% vs.\ SphereFormer's 74.8\%); yet \mymethod still \textbf{\emph{mines useful semantic guidance}} from it and outperforms the previous state-of-the art~\cite{xiong2024mining}, albeit with a modest drop. This validates that even imperfect semantics suffice for \mymethod to prevent semantic collapse, demonstrating robustness and thereby opening avenues for exploring task synergy.

\begin{figure}[tb]
  \centering
  \includegraphics[width=1.0\linewidth]{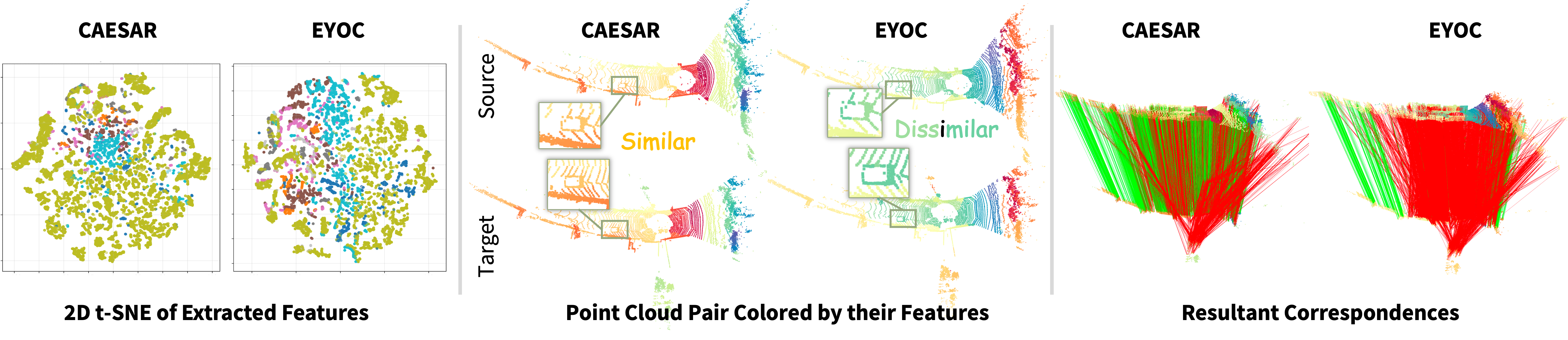}
  \caption{\textbf{(Left:) 2D t-SNE Visualization of the Source Point Cloud.} 
  Colored by Class Labels.  
  \mymethod's features are more semantically discriminative than EYOC's.
  \textbf{(Mid \& Right:) Visualization with the Corresponding Target Point Cloud.} 
  Point clouds are colored by features as done in \cite{choy2019fully}.
  \textcolor{red}{Red} lines denote outlier.
  }
  \label{fig:sem_qual}
\end{figure}

\paragraph{Qualitative Results}
We qualitatively analyze the impact of \mymethod on feature-space structure and matching quality.
In \cref{fig:sem_qual} (Left), 2D t-SNE\cite{hinton2002stochastic} visualizations colored by class labels from SemanticKITTI\cite{behley2019semantickitti} show that \mymethod's features are well separated compared to EYOC's.
Consequently, in \cref{fig:sem_qual} (Mid.), the features extracted by \mymethod are more discriminative, resulting in fewer outlier correspondences (\cref{fig:sem_qual} (Right)).
Note that EYOC produces dissimilar features for the same semantic entity (highlighted regions), whereas \mymethod's features remain consistent.
This confirms that \mymethod effectively prevents semantic collapse, producing more stable and discriminative features that improve matching quality, especially in long-range scenarios where geometric cues are weak.

\begin{figure}[tb]
\begin{minipage}[l]{0.50\linewidth}
  \centering
  \includegraphics[width=\linewidth]{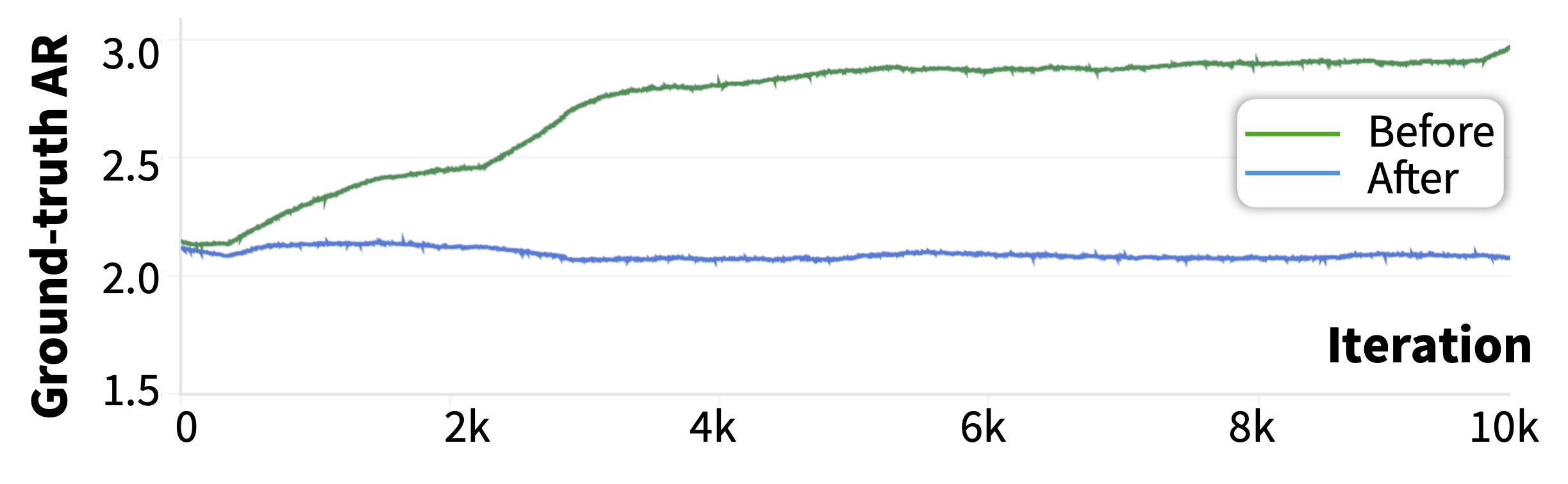}
  \captionof{figure}{\textbf{Average Rank During Training.} \adapterfull effectively lowers and stabilizes the AR of potentially correct matches during training.}
  \label{fig:ar}
\end{minipage}
\hfill
\begin{minipage}[r]{0.48\linewidth}
  \centering
  \setlength{\tabcolsep}{1.2pt}
  \renewcommand{\arraystretch}{0.9}
  \captionof{table}{\textbf{Training/Testing Time on \kitti.}  \(\Delta\%\) is the relative percentage difference with previous state-of-the-art INTEGER\cite{xiong2024mining}.}
  \scriptsize{
    \begin{tabular}{lcccc}
    \toprule
    Methods & mRR\% & \makecell{RR@\(d\in\)\\\([40,50)\)} & \makecell{Training\\(h)} & \makecell{Testing \\ (s/pair)} \\
    \midrule
    EYOC & 83.2 & 52.3 & \textbf{44.9} & \multirow{3}{*}{$\sim$0.16} \\
    \cmidrule{1-4}
    INTEGER & 84.0 & 54.2 & 76.0 & \\
    \mymethod & \textbf{86.5} & \textbf{61.3} & 62.9 & \\
    \midrule  
    \rowcolor{blue!8} \quad $\Delta\%$ & +2.9\% & +13.1\% & -17.2\% & -- \\
    \bottomrule
  \end{tabular}
  }
  \label{tbl:efficiency}
\end{minipage}
\end{figure}

\paragraph{\adapter's Effectiveness for Stability}
We evaluate \adapterfull using the \emph{average rank of ground-truth} (\(\mathrm{AR}_{\text{gt}}\)) of inliers in \cref{fig:ar}.
Its definition follows \cref{eq:ar} but is restricted to correspondences whose ground-truth inliers appear within the candidate set.
As frame distance increases during progressive training\cite{liu2024extend,xiong2024mining}, the AR before refinement (green line) rises during progressive training due to the ambiguity induced by lower overlap. However, \labeler stabilizes feature-space discriminability, as evidenced by the consistent AR after refinement (blue line).

\paragraph{Training/Testing Efficiency}
Training and testing times are shown in \cref{tbl:efficiency}.
The iterative nature and full-teacher adaptation strategy of INTEGER~\cite{xiong2024mining} make its training computationally intensive.
By alleviating the self-referential collapse via semantic anchoring, \mymethod eliminates the need for extensive iterative refinement that optimizes the full teacher model, delivering stronger results while reducing training time.
Testing efficiency matches that of the original FCGF\cite{choy2019fully} because all additional modules are discarded at inference.

\subsection{Ablation Study}

We extensively ablate \mymethod's design in \cref{tbl:ablation}.
\matcher is crucial: neither \(s_{\text{sem}}\) nor \(s_{\text{geo}}\) alone suffices, confirming the complementarity of semantic and geometric cues.
\adapterfull significantly stabilizes training; removing it degrades mRR by 1.7\%.
Among its variants, replacing \(\mathcal{L}_{\text{rerank}}\) with \(\mathcal{L}_{\text{reg}}\) causes the sharpest decline (\(-1.5\%\) mRR), confirming that the candidate-set-local surrogate of AR is critical; the global loss fails to utilize the per-set ambiguity and risks corrupting the full feature space with noise. 
\plfilterfull is essential for pseudo-label quality; its removal causes a considerable drop in tIR@1\textsuperscript{st}, triggering cascading errors in teacher-student training that are particularly severe in long-range scenarios.
Removing \dstlloss notably reduces final mRR, while replacing it with vanilla \(L_1\) loss degrades performance even further, demonstrating that na\"ive feature-space mimicry harms the student's learning.

\section{Conclusion}
\label{sec:conclusion}
In this paper, we observe that registration models spontaneously develop latent semantic structure, yet noisy pseudo-labels quickly erode it, triggering semantic collapse. \mymethod addresses this fragility by stabilizing the emergent awareness at training time through an off-the-shelf segmentation model, without adding cost at inference or requiring semantic annotations on the registration data.
For the teacher, \matcher recovers reliable correspondences and exposes local ambiguities with semantic information; \labeler then refines the feature space within each mini-batch and mines accurate pseudo-labels. For the student, \dstlloss complements standard registration supervision by encouraging holistic semantic reasoning in the learned features.
Results on \kitti and \nuscenes establish new state-of-the-art registration performance, with particularly large margins on the sparser \nuscenes benchmark, underscoring the effectiveness of semantic anchoring when geometric cues alone are insufficient.
\textbf{\emph{Limitations and future work}} are discussed in the supplementary material.

\section{Acknowledgements}
This work was supported by the National Natural Science Foundation of China (No.42571514).

%
%
\bibliographystyle{splncs04}
\bibliography{main}

\clearpage
\setcounter{page}{1}
\appendix

\begin{center}
    \vspace{0.2cm}
    \Large \textbf{Supplementary Material for \mymethodvivid}
    \vspace{0.2cm}
\end{center}
\vspace{0.5cm}

\section{Detailed Mathematical Proofs and Analysis of \texorpdfstring{$\mathcal{L}_{\text{rerank}}$}{L-rerank}}
\label{sec:proof}

In \cref{sec:reselection} of the main paper, we define the Average Rank (AR) metric (\cref{eq:ar}) as a feature-space quality diagnostic for measuring local ambiguity.
Recall that for a set of selected correspondences $\hat{\mathcal{C}}$, AR is defined as:
\begin{equation}
    \text{AR}(\hat{\mathcal{C}}) = \frac{1}{|\hat{\mathcal{C}}|} \sum_{(\mathbf{p}_i,\mathbf{q}_j) \in \hat{\mathcal{C}}} \min_k \{k \mid \mathfrak{N}^k_\mathcal{Q}(\mathbf{p}_i)= \mathbf{q}_j \}.
    \label{eq:ar_supp}
\end{equation}
A lower AR indicates that correct matches surface near the top of the candidate set $\mathcal{S}_i$ retrieved by k-NN search. 
To optimize this metric within the \labeler's \adapterfull stage, we must bridge the discrete, non-differentiable ranking operation with a continuous objective suitable for the gradient-based optimization. We propose the re-ranking loss $\mathcal{L}_{\text{rerank}}$ (\cref{eq:rerank} of the main paper) as a smooth surrogate that directly targets the AR diagnostic.
Below, we formally prove that $\mathcal{L}_{\text{rerank}}$ is an upper bound on the expected logarithmic rank of correct correspondences, and thus minimizing it effectively minimizes AR. We further establish the mathematical equivalence between our re-ranking loss and the InfoNCE\cite{oord2018representation} loss, showing why our localized formulation is particularly effective at preventing semantic collapse compared to global contrastive objectives.

\subsection{Proof for the Relaxation and Upper Bound}

\begin{proposition}
The re-ranking loss $\mathcal{L}_{\text{rerank}}$ represents a smooth, differentiable upper bound on the expectation of the logarithmic rank of the correct correspondences.
\end{proposition}
\begin{proof}
Let $s(\mathbf{p}_i, \mathbf{q}_k)$ denote the cosine similarity between the refined features of point $\mathbf{p}_i$ and a candidate $\mathbf{q}_k \in \mathcal{S}_i$. For a given selected correspondence $\mathbf{c}_i = (\mathbf{p}_i, \mathbf{q}_j) \in \hat{\mathcal{C}}$, the discrete rank of the target $\mathbf{q}_j$ within the candidate set $\mathcal{S}_i$ can be expressed using the indicator function $\mathbb{I}[\cdot]$ as:
\begin{equation}
    \text{Rank}_i(\mathbf{q}_j) = 1 + \sum_{\mathbf{q}_k \in \mathcal{S}_i \setminus \{\mathbf{q}_j\}} \mathbb{I}[s(\mathbf{p}_i, \mathbf{q}_k) > s(\mathbf{p}_i, \mathbf{q}_j)] .
\end{equation}
so that AR (\cref{eq:ar_supp}) can equivalently be written as:
\begin{equation}
    \text{AR}(\hat{\mathcal{C}}) = \frac{1}{|\hat{\mathcal{C}}|} \sum_{\mathbf{c}_i \in \hat{\mathcal{C}}} \text{Rank}_i(\mathbf{q}_j).
\end{equation}
Intuitively, this counts how many candidates have a higher similarity to the anchor than the correct match $\mathbf{q}_j$, plus one for the rank itself.
However, $\mathbb{I}[x > 0]$ is still discontinuous with zero gradients almost everywhere.
To overcome this, we introduce a smooth, strictly positive exponential upper bound. For any scalar $x \in \mathbb{R}$ and temperature $\tau > 0$, the inequality $\mathbb{I}[x > 0] \leq \exp(x/\tau)$ holds. Substituting this relaxation yields a differentiable upper bound on the rank:
\begin{equation}
    \text{Rank}_i(\mathbf{q}_j) \leq 1 + \sum_{\mathbf{q}_k \in \mathcal{S}_i \setminus \{\mathbf{q}_j\}} \exp\left(\frac{s(\mathbf{p}_i, \mathbf{q}_k) - s(\mathbf{p}_i, \mathbf{q}_j)}{\tau}\right),
\end{equation}

To prevent exponentially large gradients (which leads to instability when semantic collapse is severe), we apply $\log(\cdot)$ to both sides:
\begin{equation}
    \log(\text{Rank}_i(\mathbf{q}_j)) \leq \log \left( 1 + \sum_{\mathbf{q}_k \in \mathcal{S}_i \setminus \{\mathbf{q}_j\}} \exp\left(\frac{s(\mathbf{p}_i, \mathbf{q}_k) - s(\mathbf{p}_i, \mathbf{q}_j)}{\tau}\right) \right).
\end{equation}
This transformation also aligns the objective with information-theoretic log-likelihoods.
Finally, taking the expectation (average) over all selected correspondences $\mathbf{c}_i \in \hat{\mathcal{C}}$ yields our proposed objective:
\begin{equation}
    \begin{aligned}
      &\qquad \frac{1}{|\hat{\mathcal{C}}|} \sum_{\mathbf{c}_i \in \hat{\mathcal{C}}} \log(\text{Rank}_i(\mathbf{q}_j)) \\ 
      &\leq \frac{1}{|\hat{\mathcal{C}}|} \sum_{\mathbf{c}_i \in \hat{\mathcal{C}}} \log \left( 1 + \sum_{\mathbf{q}_k \in \mathcal{S}_i \setminus \{\mathbf{q}_j\}} \exp \left( \frac{s(\mathbf{p}_i, \mathbf{q}_k) - s(\mathbf{p}_i, \mathbf{q}_j)}{\tau} \right) \right) \\
      &= \mathcal{L}_{\text{rerank}}.
    \end{aligned}
\end{equation}
Since $\log$ is strictly monotone, minimizing this upper bound drives down the average rank itself, directly targeting $\text{AR}(\hat{\mathcal{C}})$ (\cref{eq:ar_supp}) and pulling correct matches to the top of their candidate sets.
\end{proof}

\subsection{Discussion and Relevance to InfoNCE}

Standard contrastive learning approaches (\eg Hardest Contrastive Loss\cite{choy2019fully}, InfoNCE\cite{yew2022regtr,oord2018representation}, PointInfoNCE\cite{xie2020pointcontrast}) define negatives \emph{globally} across a batch. 
Since the supervision signals mined from an unsupervised registration framework are inherently noisy, this global push can inadvertently repel geometrically and semantically similar points, degrading the feature space and leading to semantic collapse.
Here, we show that our loss derived from the AR surrogate (\cref{eq:rerank}) is equivalent to a constrained variant of the InfoNCE loss, providing complementary theoretical grounding for (1)~its optimization effectiveness and (2)~its resilience against collapse.

\begin{theorem}
The rank-surrogate $\mathcal{L}_{\text{rerank}}$ is mathematically equivalent to an InfoNCE loss computed strictly over the localized sub-space in the feature space, defined by candidate set $\mathcal{S}_i$.
\end{theorem}
\begin{proof}
Let $\mathcal{L}_{\text{NCE}}^{\text{local}}$ be an InfoNCE loss defined such that the positive pair is $(\mathbf{p}_i, \mathbf{q}_j)$ and the denominator marginalizes exclusively over the local geometric $K$-NN candidate set $\mathcal{S}_i$:
\begin{equation}
    \mathcal{L}_{\text{NCE}}^{\text{local}} = -\frac{1}{|\hat{\mathcal{C}}|}\sum_{\mathbf{c}_i\in\hat{\mathcal{C}}} \log \frac{\exp(s(\mathbf{p}_i,\mathbf{q}_j)/\tau)}{\sum_{\mathbf{q}_k \in \mathcal{S}_i} \exp(s(\mathbf{p}_i,\mathbf{q}_k)/\tau)}.
\end{equation}
Using the logarithmic quotient rule $-\log(a/b) = \log(b/a)$, we invert the fraction:
\begin{equation}
    \mathcal{L}_{\text{NCE}}^{\text{local}} = \frac{1}{|\hat{\mathcal{C}}|}\sum_{\mathbf{c}_i\in\hat{\mathcal{C}}} \log \left( \frac{\sum_{\mathbf{q}_k \in \mathcal{S}_i} \exp(s(\mathbf{p}_i,\mathbf{q}_k)/\tau)}{\exp(s(\mathbf{p}_i,\mathbf{q}_j)/\tau)} \right).
\end{equation}
We then decompose the summation in the numerator into the positive instance ($\mathbf{q}_k = \mathbf{q}_j$) and the remaining negative instances ($\mathbf{q}_k \in \mathcal{S}_i \setminus \{\mathbf{q}_j\}$):
\begin{equation}
    \mathcal{L}_{\text{NCE}}^{\text{local}} = \frac{1}{|\hat{\mathcal{C}}|}\sum_{\mathbf{c}_i\in\hat{\mathcal{C}}} \log \left( \frac{\exp(s(\mathbf{p}_i,\mathbf{q}_j)/\tau) + \sum_{\mathbf{q}_k \in \mathcal{S}_i \setminus \{\mathbf{q}_j\}} \exp(s(\mathbf{p}_i,\mathbf{q}_k)/\tau)}{\exp(s(\mathbf{p}_i,\mathbf{q}_j)/\tau)} \right).
\end{equation}
Dividing through by the denominator, we obtain:
\begin{equation}
    \mathcal{L}_{\text{NCE}}^{\text{local}} = \frac{1}{|\hat{\mathcal{C}}|}\sum_{\mathbf{c}_i\in\hat{\mathcal{C}}} \log \left( 1 + \sum_{\mathbf{q}_k \in \mathcal{S}_i \setminus \{\mathbf{q}_j\}} \exp\left(\frac{s(\mathbf{p}_i, \mathbf{q}_k) - s(\mathbf{p}_i, \mathbf{q}_j)}{\tau}\right) \right) \equiv \mathcal{L}_{\text{rerank}}.
\end{equation}
This completes the proof.
\end{proof}

\subsubsection{Discussion on Semantic Collapse.} This equivalence reveals a key structural advantage of our formulation. As noted above and in the main paper, standard, global contrastive losses push the anchor away from a broad, randomly sampled set of negatives across the batch. In outdoor LiDAR scenarios with high geometric ambiguity, this global push can inadvertently repel geometrically and semantically similar points, "starving" the feature space. 
In contrast, $\mathcal{L}_{\text{rerank}}$ acts as a \textit{localized hard-negative miner}. By constraining the denominator to $\mathcal{S}_i$, the loss exclusively penalizes the spurious nearest neighbors identified by \matcher. It concentrates representational capacity on resolving the local ambiguities that trigger semantic collapse, directly enabling the compact, batch-specific updates employed in \cref{alg:satr}.

\section{More Details on Our Key Insights}
\label{sec:key_insights}

\subsection{Experimental Protocol for Native Semantic Awareness (\cref{fig:motivation})}
To validate native semantic understanding in registration models, we employ a standard linear probing approach~\cite{xie2020pointcontrast}. We extract point features from FCGF models trained on KITTI under each unsupervised framework and freeze all backbone parameters, training only a single linear classifier to predict labels from SemanticKITTI~\cite{behley2019semantickitti}.
The evaluation protocol comprises: (i) standard train-validation split on SemanticKITTI sequences, (ii) 20-epoch training with SGD (lr=$1\mathrm{e}{-2}$, weight decay=$1\mathrm{e}{-4}$, batch size 16), and (iii) test-set mIoU across 19 semantic classes\cite{lai2023spherical}. The second panel of \cref{fig:motivation} assesses robustness against label corruption by injecting controlled fractions of random correspondence outliers and measuring the resulting degradation in both segmentation and registration performance.

\subsection{Statistical Validation of Correspondence Reselection (\cref{fig:topk_recall})}
We empirically justify \matcher's reselection strategy via statistical analysis on \kitti, using features from early training stages where ambiguity is most pronounced. We measure the fraction of source points whose true correspondence appears within the retrieved candidate set.
Formally, consider the overlap region $\mathcal{P}_{ov} \subset \mathcal{P}$ containing source points with valid correspondences in target $\mathcal{Q}$ under ground-truth transformation $\mathbf{T}_{gt}$ and distance threshold $\tau=0.3$m. For each $p_i \in \mathcal{P}_{ov}$, we retrieve top-$K$ feature-space neighbors $\mathfrak{N}_\mathcal{Q}^K(p_i)$ and compute:
\begin{equation}
    \rho_K = \frac{1}{|\mathcal{P}_{ov}|} \sum_{p_i \in \mathcal{P}_{ov}} \mathbb{I} \left[ \min_{q_j \in \mathfrak{N}_\mathcal{Q}^K(p_i)} \| \mathbf{T}_{gt}(p_i) - q_j \|_2 < \tau \right].
\end{equation}

The cumulative distribution analysis in \cref{fig:topk_recall} reveals a critical insight: while na\"ive 1-NN matching achieves low matching recall ($\rho_1 < 0.5$ for most pairs), expanding to 5-NN dramatically improves coverage ($\rho_5 > 0.5$ for the majority of pairs). This validates our core hypothesis that correct correspondences frequently exist within candidate sets but are obscured by feature-space noise, motivating \matcher's reselection strategy over the rejection-only designs of existing methods~\cite{liu2024extend,xiong2024mining}.

\section{Implementation Details}
\label{sec:hyperparameter}

\subsection{Implementation Details}
We implement \mymethod{} using PyTorch~\cite{paszke2019pytorch}. Training requires only a \emph{single} NVIDIA RTX 3090 (24GB). Our implementation leverages MinkowskiEngine~\cite{choy20194d} for sparse convolutions, consistent with established benchmarks~\cite{liu2024extend,xiong2024mining}.
All semantic components in \mymethod{} operate exclusively at training time. During inference, $\mathcal{F}_{\text{sem}}$, \adapterfull, and \matcher{} are entirely discarded, leaving only the trained FCGF student. Inference thus runs at the same 0.16\,s/pair cost as vanilla FCGF~\cite{choy2019fully}.

\subsection{Optimization and Hyperparameters}
Optimization strategies follow each model's original configuration to ensure fair comparison, \eg, SGD for FCGF~\cite{choy2019fully} with momentum 0.9 and weight decay $1\mathrm{e}{-4}$. The per-batch \adapter learning rate of 0.1 with OneCycleLR scheduling enables rapid per-batch convergence. We train for 210 epochs on KITTI and 240 epochs on nuScenes. The batch size is 6 and 8 for \kitti and \nuscenes, respectively.
Key hyperparameter choices include: a candidate set size of $K=5$ in \matcher, providing an optimal balance between computational cost and correspondence quality; geometric saliency scores computed using a neighborhood of $k=48$, capturing sufficient local structure; and an \adapterfull iteration limit of 5, ensuring convergence while preventing overfitting. The exploration weight $\gamma=0.001$ encourages diversity without destabilizing training, while temperature $\tau=0.7$ in re-ranking provides appropriate softmax sharpening. To determine whether AR (\cref{eq:ar}) has stabilized in the EM-style optimization, we use a difference threshold of 0.005. The EMA decay rate follows existing efforts\cite{liu2024extend,xiong2024mining}.

\section{Architecture Details}
\label{sec:architecture}

\begin{table}[tb]
  \centering
  \setlength{\tabcolsep}{4pt}
  \caption{Parameter Comparison between \adapter{} and Full FCGF Teacher}
  \label{tbl:adapter_eff}
  \small{\begin{tabular}{p{3cm}lc}
    \toprule
    Model & \#Param (M) & Model Size vs. FCGF \\
    \midrule
    Full FCGF Teacher & 8.76 & 100\% \\
    \adapter-only in \mymethod & 0.07 & 0.8\% \\
    \bottomrule
  \end{tabular}}
\end{table}

\subsection{Architectural Details for \adapterfull}
\adapterfull employs a lightweight architecture for efficient per-batch optimization. The semantic residual extractor $\Phi_g$ leverages FCGF's ResBlock~\cite{choy2019fully} for aggregating semantic guidance, while the point-wise refiner $\Phi_r$ uses a compact two-layer ReLU-activated MLP, ensuring rapid convergence with standard SGD while preventing overfitting to noise.

\subsection{Computational Efficiency}
Teacher refinement requires balancing adaptation capability with computational overhead. Existing approaches either freeze the teacher entirely~\cite{liu2024extend}, leaving noisy features unchecked, or adapt the full teacher per batch~\cite{xiong2024mining}, incurring prohibitive cost. Our lightweight module leverages semantic--geometric interactions while preserving computational tractability.
\cref{tbl:adapter_eff} quantifies the efficiency: \adapterfull requires only 0.07M trainable parameters per batch, a $125\times$ reduction compared to full teacher adaptation (8.76M for FCGF). The two-level stop-gradient design (\cref{eq:recon_loss,eq:low} in the main paper) further improves training stability by letting each component specialize without conflicting gradients.

\section{Dataset and Evaluation Details}
\label{sec:dataset}

\subsection{Dataset}
Our experiments use two standard large-scale outdoor datasets widely adopted for point cloud registration~\cite{lu2021hregnet,liu2023regformer,zeng2025unlocking}, especially in the unsupervised setting\cite{liu2024extend,xiong2024mining}.

\begin{itemize}
    \item \textbf{KITTI Odometry}~\cite{geiger2012we} provides high-resolution scans (64-beam Velodyne HDL-64E) across 11 sequential trajectories covering diverse driving environments.
    Following the standard protocol~\cite{lu2021hregnet,qin2022geometric,zeng2025unlocking,liu2023regformer,huang2021predator,choy2019fully,bai2020d3feat}, we use sequences 00--05 for training, 06--07 for validation, and 08--10 for testing.
    For semantic evaluation, we leverage labels from SemanticKITTI\cite{behley2019semantickitti}, and the annotations are typically mapped into 19 categories \cite{lai2023spherical,tang2020searching,choy20194d}.
    
    \item \textbf{nuScenes}~\cite{caesar2020nuscenes} contains 1000 scenes, with 850 used for training/validation and 150 for testing, following existing work~\cite{lu2021hregnet,huang2021predator,liu2023regformer,zeng2025unlocking}.
    It provides point clouds from a 32-beam Velodyne sensor and annotations for 16 distinct object categories~\cite{lai2023spherical,tang2020searching,choy20194d}. The reduced LiDAR resolution combined with complex urban scenes creates substantial geometric ambiguity that stress-tests registration methods.
\end{itemize}

\section{Limitations and Future Directions}
\label{sec:limitations}
While \mymethod{} eliminates inference overhead, it introduces a training-time dependency on a pretrained, off-the-shelf 3D segmentation model. We mitigate potential information leakage by using cross-dataset segmentation models by default, and demonstrate in \cref{tbl:scalessm} (main paper) that even unsupervised sim-to-real adapted semantics (DGT-ST, with lower mIoU) still yield improvements over prior state-of-the-art methods.
Several promising directions remain for future work:
\begin{itemize}[leftmargin=1.5em,itemsep=2pt,parsep=0pt]
  \item \textbf{Scaling to larger registration models.} Effectively training larger models such as Predator is left for future work.
  \item \textbf{Extending to cross-modal tasks.} The semantic collapse phenomenon may also affect cross-modal (\eg, point cloud to image) registration; investigating whether our anchoring strategy generalizes to these settings is an interesting future direction.
\end{itemize}

\end{document}